\documentclass[conference]{IEEEtran}
\IEEEoverridecommandlockouts
\usepackage{cite}
\usepackage{amsmath,amssymb,amsfonts}
\usepackage{algorithmic}
\usepackage{graphicx}
\usepackage{textcomp}
\usepackage{xcolor}

\usepackage{amsfonts}
\usepackage{algorithm}
\usepackage{amsmath}
\usepackage{multirow}
\usepackage{array}
\usepackage{amssymb} 

\usepackage{amsthm}

\newtheorem{proposition}{Proposition}

\def\b{\ensuremath\boldsymbol}

\usepackage{url}
\usepackage[algo2e,linesnumbered,lined,boxed,vlined]{algorithm2e}

\usepackage{balance}

\usepackage{eso-pic} 

\def\BibTeX{{\rm B\kern-.05em{\sc i\kern-.025em b}\kern-.08em
    T\kern-.1667em\lower.7ex\hbox{E}\kern-.125emX}}
\begin{document}
\bstctlcite{IEEEexample:BSTcontrol}

\AddToShipoutPictureBG*{%
  \AtPageUpperLeft{%
    \setlength\unitlength{1in}%
    \hspace*{\dimexpr0.5\paperwidth\relax}
    \makebox(0,-0.75)[c]{\normalsize Accepted for presentation at the IEEE International Conference on Systems, Man, and Cybernetics (SMC) 2020}
    }}


\title{Isolation Mondrian Forest \\for Batch and Online Anomaly Detection}



\author{
\IEEEauthorblockN{Haoran
Ma$^{1,2,*}$, Benyamin Ghojogh$^{1,*}$, Maria N.
Samad$^{1,*}$, Dongyu Zheng$^{1,3}$, Mark Crowley$^1$\thanks{$^*$The first three authors contributed equally to this work.}
}
\IEEEauthorblockA{\\$^1$Department of Electrical and Computer Engineering, University of Waterloo, Waterloo, ON, Canada \\
$^2$SoundHound, Toronto, ON, Canada \\ $^3$Facebook, Seattle, Washington, USA  \\
Emails: h65ma@uwaterloo.ca, bghojogh@uwaterloo.ca, mnsamad@uwaterloo.ca,\\ dongyu.zheng@edu.uwaterloo.ca, mcrowley@uwaterloo.ca
}}

\maketitle

\begin{abstract}
We propose a new method, named isolation Mondrian forest (iMondrian forest), for batch and online anomaly detection. The proposed method is a novel hybrid of isolation forest and Mondrian forest which are existing methods for batch anomaly detection and online random forest, respectively. iMondrian forest takes the idea of isolation, using the depth of a node in a tree, and implements it in the Mondrian forest structure. The result is a new data structure which can accept streaming data in an online manner while being used for anomaly detection. Our experiments show that iMondrian forest mostly performs better than isolation forest in batch settings and has better or comparable performance against other batch and online anomaly detection methods.
\end{abstract}


\section{Introduction}

Anomaly detection \cite{chandola2009anomaly} refers to the task of detecting the outliers in a dataset where the anomalies are different from the regular pattern of data. 
Anomaly detection can be performed on either offline or online data where the data are processed as a batch or stream, respectively. 
Numerous methods have been proposed for batch and online anomaly detection for use in a multitude of applications such as fraud detection, disease diagnosis, and intrusion detection. 

One batch anomaly detection method is Local Outlier Factor (LOF) \cite{breunig2000lof} which defines a measure for local density of every data point according to its neighbors and then compares the local density of every point with its neighbors to find the anomalies. 
One-class SVM \cite{scholkopf2000support} is another method which estimates a function to be $1$ and $-1$ in the regions with high and low density of data, respectively. It can also use kernels to map the data to a higher dimensional feature space for a possible better performance. 
Isolation forest (iForest)  \cite{liu2008isolation} is an isolation-based method \cite{liu2012isolation} which isolates anomalies rather than separating normal points. Its main idea is that anomalies are separated shallower in the tree by a random forest so the depth of a node can determine the anomaly score of a point.

One method for online anomaly detection is incremental LOF \cite{pokrajac2007incremental}, which updates the local density and other characteristics of the LOF algorithm for the new data points in the stream. It also updates the density for the existing points which are affected by the new data in their $k$-nearest neighbors. 
There is another method which uses kernel density estimation for online anomaly detection which assigns the anomaly label to a point if it deviates significantly from the estimated density \cite{ahmed2009online}; note that, this method does not calculate scores for the points.
Oversampling Principal Component Analysis (osPCA) is an online anomaly detection method which oversamples a point and calculates the principal direction both with and without the oversampled point. If the principal direction deviates significantly, the point is considered to be anomalous. There exist two versions of osPCA: osPCA1 with power method \cite{yeh2009anomaly} and osPCA2 with least squares approximation \cite{lee2013anomaly}.

Decision forests \cite{criminisi2012decision} are ensembles of decision trees which partition the input space for different tasks such as classification and regression.
Random forest \cite{breiman2001random} is a forest of binary trees which randomly sample from data points and features for the different trees. As the trees in random forest are not very correlated, the variance of estimation is reduced because of bootstrap aggregating or bagging \cite{ghojogh2019theory}.
Extremely randomized trees \cite{geurts2006extremely} select both the split dimension and split value randomly. 
Some forests, such as Hoeffding trees \cite{domingos2000mining}, are proposed for online processing of data \cite{abdulsalam2007streaming,saffari2009line}.  
The Hoeffding tree learns a regression tree model but stores no data points along the way. New points help build confidence in the split at each node, and once the confidence is high enough, the node is split. Later arriving points are used to start learning the new leaves which could in turn become new internal nodes. 
Mondrian forest (MForest) is another online forest method for classification \cite{lakshminarayanan2014mondrian} and regression \cite{lakshminarayanan2016mondrian}.
It is based on the Mondrian processes \cite{roy2008mondrian} which are a family of distributions over tree structures. 
The trees in MForest can grow incrementally without complex tree rotation and correction which previous streaming tree methods required. The new nodes in MForest can also be added as internal nodes. 
While Hoeffding trees perform online learning and scale well for regression and classification, their particular incremental approach to tree building and lack of restructuring makes them less suitable than MForest to utilize the isolation concept \cite{liu2012isolation} for anomaly detection.
In this paper we propose \textit{isolation Mondrian forests} (\textit{iMondrian forest} or \textit{iMForest}), a novel hybrid of iForest and MForest providing the best of both worlds: an anomaly detector which can perform unsupervised learning of anomalous and normal points for both batch and online processing. 

Assume we have a batch of data denoted by $\mathcal{X} := \{\b{x}_i\}_{i=1}^n$ where $n$ is the sample size and $d$ is the dimensionality of data, i.e., $\b{x}_i \in \mathbb{R}^d$. We may also have $m$ new data points denoted by $\mathcal{X}^{(n)} := \{\b{x}_i^{(n)}\}_{i=1}^m$. Our goal is to detect the anomalies in $\mathcal{X}$ and $\mathcal{X}^{(n)}$ using both batch and online processing of data. 

\section{Background}

In this section, we review isolation forest \cite{liu2008isolation} and Mondrian forest \cite{lakshminarayanan2014mondrian} in more detail.

\subsection{Isolation Forest}\label{section_isolattion_forest}

An \textit{isolation forest} (\textit{iForest}) \cite{liu2008isolation} is an ensemble of isolation trees. An \textit{isolation tree} is an extremely randomized tree \cite{geurts2006extremely} where the tree is a proper binary tree and its splitting dimension $q$ and splitting value $p$ are randomly selected at every node. The tree grows until every leaf includes exactly one data point, i.e., $|\mathcal{X}|=1$ in the leaf node.
Let $h(\b{x})$ denote the path length for a data point $\b{x}$ in the tree where the path length is defined as the number of edges $\b{x}$ traverses from the root to the leaf it belongs to. The average height of an isolation tree is $\log(n)$. As the structure of the isolation tree is equivalent to the binary search tree, the estimation of average path length in isolation trees is \cite{bruno2000data}: 
\begin{align}\label{equation_c}
c(n) := 2\, h(n-1) - \big(2\,(n-1)/n\big),
\end{align}
where $h(i)$ is the $i$-th harmonic number, defined as: 
\begin{align}
h(i) := \ln(i) + 0.5772156649,
\end{align}
where the added constant is the Euler's constant. 
The anomaly score of a point $\b{x}$ is:
\begin{align}\label{equation_anomaly_score}
s(\b{x}) := 2^{-\mathbb{E}(l(\b{x})) / c(n)},
\end{align}
where $\mathbb{E}(l(\b{x}))$ is the expected path length for the data point $\b{x}$ among the trees of the forest:
\begin{align}\label{equation_expected_path_length}
\mathbb{E}(l(\b{x})) := \frac{1}{|\mathcal{F}|} \sum_{t=1}^{|\mathcal{F}|} l_t(\b{x}),
\end{align}
where $l_t(\b{x})$ is the path length of $\b{x}$ in the $t$-th tree and $|\mathcal{F}|$ is the population of trees in the forest.
The intuition of anomaly score in iForests is that the anomalies tend to be isolated sooner, i.e., shallower in the tree, while the normal points require more splits to become isolated, i.e., deeper the tree.
The anomaly score is in the range $s(\b{x}) \in [0, 1]$ where $s(\b{x}) = 0$ and $s(\b{x}) = 1$ corresponded to normal and anomaly points, respectively. The $s=0.5$ may be a proper threshold for anomaly detection in iForests. 
Note that the authors of the original iForest use a subset of the data with subsampling size specified by $\psi = 256$ \cite{liu2008isolation}.

\subsection{Mondrian Forest}

A \textit{Mondrian forest} \cite{lakshminarayanan2014mondrian} is an ensemble of Mondrian trees which are based on the Mondrian process \cite{roy2008mondrian}. \textit{Mondrian processes} are families of random hierarchical binary partitions and probability distributions over tree data structures. While Mondrian processes are infinite structures, Mondrian trees are restrictions of Mondrian processes on a finite set of points. 
Every node $r$ in the Mondrian tree has a \textit{split time} $\tau_r$ which increases with the depth of the node. The split time is zero at the root and infinite at the leaves of the tree.

Let $\hat{\mathcal{B}}_r := (\hat{\ell}_{r1}, \hat{u}_{r1}] \times \dots \times (\hat{\ell}_{rd}, \hat{u}_{rd}]$ for the $r$-th node, where $\hat{\ell}_{rj}$ and $\hat{u}_{rj}$ are the lower and upper bounds of hyper-rectangular block $\hat{\mathcal{B}}_r$ along dimension $j$. The Mondrian tree considers the smallest block containing the data points in a node; therefore, it defines $\mathcal{B}_r := (\ell_{r1}, u_{r1}] \times \dots \times (\ell_{rd}, u_{rd}]$ where $\ell_{rj}$ and $u_{rj}$ are the lower and upper bounds of the smallest hyper-rectangular block $\mathcal{B}_r$ along dimension $j$. For a node indexed by $r$, let $\b{\ell}_{\mathcal{X}_b} = [\ell_{r1}, \dots, \ell_{rd}]^\top$ and $\b{u}_{\mathcal{X}_b} = [u_{r1}, \dots, u_{rd}]^\top$; thus, $\b{\ell}_{\mathcal{X}_b} := \text{min}(\{\b{x}_i^{(b)} ~ | ~ \forall i\})$ and $\b{u}_{\mathcal{X}_b} := \text{max}(\{\b{x}_i^{(b)} ~ | ~ \forall i\})$ where $\mathcal{X}_b = \{\b{x}_i^{(b)}\} = \{\b{x}_i ~ | ~ \b{x}_i \in \mathcal{B}_r\}$.
For the $r$-th node, the split time of a node is determined as $\tau_{\text{parent}(j)} + e$ where $e$ is a random variable from an exponential distribution with a rate which is a function of $\b{\ell}_{\mathcal{X}_b}$ and $\b{u}_{\mathcal{X}_b}$.
Depending on whether the split time of the node is smaller or greater than the split time of its parent, it is put before or after the parent node in the tree. 

\SetAlCapSkip{0.5em}
\IncMargin{0.8em}
\begin{algorithm2e}[!t]
\DontPrintSemicolon
    \textbf{Procedure: } BatchTraining($\mathcal{X}$, $|\mathcal{F}|$)\;
    \textbf{Input: } $\mathcal{X} = \{\b{x}_i\}_{i=1}^n$, $|\mathcal{F}|$: number of trees \;
    \For{tree $t$ from $1$ to $|\mathcal{F}|$}{
        $\mathcal{F}$ $\leftarrow$ $\mathcal{F}$ $\cup$ iMondrianTree(root, $\mathcal{X}, 0$)\;
    }
    \textbf{Return} Forest $\mathcal{F}$\;
\caption{Batch training in iMondrian forest.}\label{algorithm_batch_training}
\end{algorithm2e}
\DecMargin{0.8em}

\SetAlCapSkip{0.5em}
\IncMargin{0.8em}
\begin{algorithm2e*}[!t]
\DontPrintSemicolon
    \textbf{Procedure: } iMondrianTree($r$, $\mathcal{X}, \tau_\text{parent}$)\;
    \textbf{Input: } $r$: node pointer, $\mathcal{X} = \{\b{x}_i\}_{i=1}^n$, $\tau_\text{parent}$: split time of the parent node \;
    $\mathcal{X}_b = \{\b{x}_i^{(b)}\} = \{\b{x}_i ~ | ~ \b{x}_i \in \mathcal{B}_r\}$\;
    $\b{\ell}_{\mathcal{X}_b} \leftarrow \text{min}(\{\b{x}_i^{(b)} ~ | ~ \forall i\})$\;
    $\b{u}_{\mathcal{X}_b} \leftarrow \text{max}(\{\b{x}_i^{(b)} ~ | ~ \forall i\})$\;
    \uIf{$|\mathcal{X}_b| > 1$}{
        $e \sim \text{Exp}\big(\lambda = \sum_{j=1}^d (\b{u}_{\mathcal{X}_b}(j) - \b{\ell}_{\mathcal{X}_b}(j))\big)$\;
        $\tau \leftarrow \tau_\text{parent} + e$\;
        $q \leftarrow $ sample from $\{1, \dots, d\}$ with distribution $\propto (\b{u}_{\mathcal{X}_b}(j) - \b{\ell}_{\mathcal{X}_b}(j))$ for the $j$-th dimension\;
        $p \sim U(\b{\ell}_{\mathcal{X}_b}(q), \b{u}_{\mathcal{X}_b}(q))$\;
        $\mathcal{X}_\text{left} \leftarrow \{\b{x} \in \mathcal{X}_b ~ | ~ \b{x}(q) < p\}$\;
        $\mathcal{X}_\text{right} \leftarrow \{\b{x} \in \mathcal{X}_b ~ | ~ \b{x}(q) \geq p\}$\;
        Left $\leftarrow$ iMondrianTree(leftChild($r$), $\mathcal{X}_\text{left}, \tau$)\;
        Right $\leftarrow$ iMondrianTree(rightChild($r$), $\mathcal{X}_\text{right}, \tau$)\;
        \textbf{Return} internalNode\{leftChild: Left, rightChild: Right, splitDim: q, splitVal: p, time: $\tau$, $\text{dim}_\text{min}$: $\b{\ell}_{\mathcal{X}_b}$, $\text{dim}_\text{max}$: $\b{u}_{\mathcal{X}_b}$, population: $|\mathcal{X}_b|$\}\;
    }
    \Else{
        \textbf{Return} leafNode\{time: $\infty$, $\text{dim}_\text{min}$: $\b{\ell}_{\mathcal{X}_b}$, $\text{dim}_\text{max}$: $\b{u}_{\mathcal{X}_b}$, population: $1$\}\;
    }
\caption{Constructing iMondrian tree.}\label{algorithm_iMondrian_tree}
\end{algorithm2e*}
\DecMargin{0.8em}

Mondrian trees can be updated with new data making them suitable for online streaming domains. When a new data point arrives, it is checked whether it belongs to an existing block or not. If not, the lower and upper errors (deviations) from the block are calculated. Again, a random variable is sampled from an exponential distribution with a rate which is a function of the lower and upper errors. As before, the split time of the new node is calculated and depending on it, its location in the tree is determined. In this way, new nodes can be added in the middle of a tree and not just grown at the end of tree like in regular online random forests.
Note that the Mondrian forest is designed for classification \cite{lakshminarayanan2014mondrian} and regression \cite{lakshminarayanan2016mondrian} so its authors propose an analysis for smooth posterior updates in the blocks. However, the Mondrian forest can be used for unsupervised purposes where the posterior analysis and the pausing process can be bypassed. This is useful for this work because anomaly detection is usually considered as an unsupervised task.

\section{iMondrian Forest}

The iMondrian forest can be used for both batch and online anomaly detection. In the following, we cover both cases.

\subsection{Batch Processing}

\subsubsection{Training}

The iMondrian forest is an ensemble of iMondrian trees. Algorithm \ref{algorithm_batch_training} shows this ensemble where $\mathcal{X} := \{\b{x}_i\}_{i=1}^n$ is the batch of data and $|\mathcal{F}|$ is the number of trees in the forest.  
Inspired by \cite{liu2008isolation}, the data in a batch can be subsampled with subsampling size $\psi = 256$ for growing the tree. If subsampling is used, $\mathcal{X}$ denotes the sample of data and $n = \psi$. 
The iMondrian tree is grown recursively as detailed in Algorithm \ref{algorithm_iMondrian_tree}. 
As with Mondrian trees, bounds of hyper-rectangular blocks are defined $\mathcal{B}_r := (\ell_{r1}, u_{r1}] \times \dots \times (\ell_{rd}, u_{rd}]$ along each of $d$ dimensions for the $r$-th node.
Let $\mathcal{X}_b := \{\b{x}_i^{(b)}\} := \{\b{x}_i ~ | ~ \b{x}_i \in \mathcal{B}_r\}$ be the subset of data which exist in the smallest block enclosing the node.
For a node, the lower and upper bounds of $\mathcal{B}_r$ along the features are denoted by $\b{\ell}_{\mathcal{X}_b}$ and $\b{u}_{\mathcal{X}_b}$, respectively. 

In order to split a block, we sample a random variable $e$ from an exponential distribution with the rate $\lambda = \sum_{j=1}^d (\b{u}_{\mathcal{X}_b}(j) - \b{\ell}_{\mathcal{X}_b}(j))$ which is the linear dimension of $\mathcal{B}_r$. We set the split time of a node to the split time of its parent plus $e$. 
We sample the dimension of the split, $q$, from a discrete distribution proportional to $(\b{u}_{\mathcal{X}_b}(j) - \b{\ell}_{\mathcal{X}_b}(j))$. 
We sample the value of the split, $p$, from a continuous uniform distribution $U(\b{\ell}_{\mathcal{X}_b}(q), \b{u}_{\mathcal{X}_b}(q))$. 
The tree is grown until every node contains a single data point, i.e., $|\mathcal{X}|=1$.

\SetAlCapSkip{0.5em}
\IncMargin{0.8em}
\begin{algorithm2e}[!t]
\DontPrintSemicolon
    \textbf{Procedure: } PathLength($\b{x}$, $t$, $l$)\;
    \textbf{Input: } $\b{x}$: data point, $t$: iMondrian tree, $l$: current path length (initialized to $0$) \;
    $q \leftarrow t$.splitDim\;
    $p \leftarrow t$.splitVal\;
    \uIf{$\b{x}(q) < p$}{
        \textbf{Return} PathLength($\b{x}$, $t$.leftChild, $l+1$)\;
    }
    \Else{
        \textbf{Return} PathLength($\b{x}$, $t$.rightChild, $l+1$)\;
    }
\caption{Calculation of path length.}\label{algorithm_path_length}
\end{algorithm2e}
\DecMargin{0.8em}

\SetAlCapSkip{0.5em}
\IncMargin{0.8em}
\begin{algorithm2e}[!t]
\DontPrintSemicolon
    \textbf{Procedure: } ExtendIMondrianForest($\mathcal{X}^{(n)}$, $\mathcal{F}$)\;
    \textbf{Input: } $\mathcal{X}^{(n)} = \{\b{x}_i^{(n)}\}_{i=1}^m$: new data, $\mathcal{F}$: Forest \;
    \For{$\b{x}_i^{(n)} \in \mathcal{X}^{(n)}$}{
        \For{tree $t \in \mathcal{F}$}{
            ExtendIMondrianTree($t$.root, $\b{x}_i^{(n)}, 0$)\;
        }
    }
\caption{Extension of iMondrian forest.}\label{algorithm_extend_iMondrian_forest}
\end{algorithm2e}
\DecMargin{0.8em}

\SetAlCapSkip{0.5em}
\IncMargin{0.8em}
\begin{algorithm2e*}[!t]
\DontPrintSemicolon
    \textbf{Procedure: } ExtendIMondrianTree($r, \b{x}^{(n)}, \tau_\text{parent}$)\;
    \textbf{Input: } $r$: node pointer, new data point: $\b{x}^{(n)}$, $\tau_\text{parent}$: split time of the parent node\;
    $\b{e}_\ell \leftarrow \text{max}(r.\text{dim}_\text{min} - \b{x}^{(n)}, \b{0})$\;
    $\b{e}_u \leftarrow \text{max}(\b{x}^{(n)} - r.\text{dim}_\text{max}, \b{0})$\;
    $e \sim \text{Exp}\big(\lambda = \sum_{j=1}^d (\b{e}_\ell(j) + \b{e}_u(j))\big)$\;
    \uIf{$\tau_\text{parent} + e < r.\tau$}{
        $q \leftarrow $ sample from $\{1, \dots, d\}$ with distribution $\propto (\b{e}_\ell(j) + \b{e}_u(j))$ for the $j$-th dimension\;
        \uIf{$\b{x}^{(n)}(q) > r.\text{dim}_\text{max}(q)$}{
            $p \sim U\big(r.\text{dim}_\text{max}(q), \b{x}^{(n)}(q)\big)$\;
        }
        \ElseIf{$\b{x}^{(n)}(q) < r.\text{dim}_\text{min}(q)$}{
            $p \sim U\big(\b{x}^{(n)}(q), r.\text{dim}_\text{min}(q)\big)$\;
        }
        newNode $\leftarrow$ internalNode\{splitDim: $q$, splitVal: $p$, time: $\tau_\text{parent} + e$, $\text{dim}_\text{min}$: min$(r$.$\text{dim}_\text{min}, \b{x}^{(n)})$, $\text{dim}_\text{max}$: max$(r$.$\text{dim}_\text{max}, \b{x}^{(n)})$, population: $r$.population + 1\}\;
        newNode.parent $\leftarrow$ $r$.parent\;
        \uIf{$\b{x}^{(n)}(q) > p$}{
            newNode.leftChild $\leftarrow r$\;
            newNode.rightChild $\leftarrow$ iMondrianTree (rightSibling($r$), $\b{x}^{(n)}$, newNode.time)\;
        }
        \Else{
            newNode.leftChild $\leftarrow$ iMondrianTree (leftSibling($r$), $\b{x}^{(n)}$, newNode.time)\;
            newNode.rightChild $\leftarrow r$\;
        }
    }
    \Else{
        $r.\text{dim}_\text{min} \leftarrow \text{min}(r.\text{dim}_\text{min}, \b{x}^{(n)})$\;
        $r.\text{dim}_\text{max} \leftarrow \text{max}(r.\text{dim}_\text{max}, \b{x}^{(n)})$\;
        \uIf{$\b{x}(r.\text{splitDim}) \leq r.\text{splitVal}$}{
            ExtendIMondrianTree($r.\text{leftChild}, \b{x}^{(n)}, r.\tau$)\;
        }
        \Else{
            ExtendIMondrianTree($r.\text{rightChild}, \b{x}^{(n)},$ $r.\tau$)\;
        }
    }
\caption{Extension of iMondrian tree.}\label{algorithm_extend_iMondrian_tree}
\end{algorithm2e*}
\DecMargin{0.8em}

\subsubsection{Evaluation}
After growing the iMondrian trees in the forest, we calculate the path length of every tree for a data point $\b{x}$ as in Algorithm \ref{algorithm_path_length}. The path length for the $t$-th tree, $l_t(\b{x})$, is the number of edges traversed by the point from the root to the node containing point $\b{x}$. 
We calculate the expected path length in the iMondrian forest using Eq. (\ref{equation_expected_path_length}) and the anomaly score for point $\b{x}$ using Eq. (\ref{equation_anomaly_score}).
For determining whether a point in the dataset is normal or an anomaly, we can either use the threshold $s=0.5$ as in \cite{liu2008isolation} or K-means clustering. In the threshold approach, the point is determined as anomaly if $s(\b{x}) > 0.5$. In the K-means approach, we assign the scores of training data into two clusters and take the points in the cluster with greater mean as the anomaly points. 
The theoretical reason for threshold $0.5$ is that the expected path length for the data point (Eq. (\ref{equation_expected_path_length})) is the estimation of the average path length (Eq. (\ref{equation_c})) when $s=0.5$ (see {\cite[p. 415]{liu2008isolation}}, same holds for us). The empirical reason is that the results of $s=0.5$ and K-means are almost the same (see Fig \ref{figure_synthetic_batch}). 

In batch processing, for an out-of-sample data point, or novelty detection \cite{pimentel2014review}, we feed the data point to the trees of iMondrian forest and calculate the score using Eq. (\ref{equation_anomaly_score}). Then, we can use the threshold $s=0.5$ again or assign the point to the cluster whose mean is closer to the score of the point. 
Our experiments showed that both the threshold and clustering approaches have almost equally good performance for batch processing.

\subsection{Online Processing}

\subsubsection{Training}

A major advantage of iMondrian forests is their ability to be updated online for new data. Let $\mathcal{X}^{(n)} := \{\b{x}_i^{(n)}\}_{i=1}^m$ denote the $m$ new data points. We process data points one-by-one to extend each tree in the forest (see Algorithm \ref{algorithm_extend_iMondrian_forest}). 
Algorithm \ref{algorithm_extend_iMondrian_tree} describes how we extend each iMondrian tree for $\b{x}^{n}$. The tree is extended recursively starting from the root. The lower and upper errors of deviation of a point from the smallest block contained by the node $r$ are calculated as $\mathbb{R}^d \ni \b{e}_\ell := \text{max}(r.\text{dim}_\text{min} - \b{x}^{(n)}, \b{0})$ and $\mathbb{R}^d \ni \b{e}_u := \text{max}(\b{x}^{(n)} - r.\text{dim}_\text{max}, \b{0})$, respectively, where $\text{dim}_\text{min}$ and $\text{dim}_\text{max}$ are the upper and lower bounds of the block along different dimensions. We sample a random variable $e$ from an exponential distribution with the rate $\lambda = \sum_{j=1}^d (\b{e}_\ell(j) + \b{e}_u(j))$. 

In the case where the split time of the node $r$ is greater than the split time of its parent plus $e$, a new node is created above the node $r$. Note that we started from the root and are moving downwards so the new node is added before the current node for which a condition holds. 
In this case, we randomly pick a split dimension $q$ from the distribution proportional to $(\b{e}_\ell(j) + \b{e}_u(j))$. 
If the value on dimension $q$ of the data point is \textit{greater} than the upper bound of the current block, then the split value $p$ is sampled from the uniform distribution $U\big(r.\text{dim}_\text{max}(q), \b{x}^{(n)}(q)\big)$.
If the value is \textit{lower}, then $p$ is sampled from $U\big(\b{x}^{(n)}(q), r.\text{dim}_\text{min}(q)\big)$.
Depending on the split value and the feature of data point, we create an iMondrian tree as the left or right sibling of the node $r$.

In the case where the split time of the node $r$ is less than the split time of its parent plus $e$, we simply descend down the tree and call the extending function recursively for the left or right of the node $r$ depending on the split dimension and split values of the children.

\subsubsection{Evaluation}

After the extension of the trees of iMondrian forest, we can process data points through the forest to calculate their anomaly scores using Eq. (\ref{equation_anomaly_score}). 
This can be done for all the new points and any other out-of-sample points. 
Whenever the trees have been updated we should also ideally process previous batches of data through the forest again to recalculate their anomaly scores. This is expected since more data will lead to an improved model and a better structure for detection of false negative or positive points. However, for performance reasons, in practice this recalculation of scores could be done for just a window of the latest points. 
For online processing, our experiments showed that the threshold $s=0.5$ is not necessarily the best threshold and K-means clustering works more better. Hence, we use K-means to cluster all the into two clusters and set the cluster with greater mean as anomalous. The out-of-sample points are assigned to the cluster whose mean is closer to their score.

\section{Complexity Analysis}

\begin{proposition}
The training and evaluation (score calculation) phases of batch processing in iMondrian forest each take $\mathcal{O}(|\mathcal{F}|\, d\, n \log n)$ and $\mathcal{O}(|\mathcal{F}|\, n \log n)$ time, respectively, assuming that the trees are balanced.
\end{proposition}

\begin{proof}
Assuming that the tree is balanced, in training, the $i$-th point traverses $\log n$ edges to reach its leaf. 
For $n$ points, we have $\mathcal{O}(\sum_{i=1}^n \log n) = \mathcal{O}(\log (n^n)) = \mathcal{O}(n \log n)$ \cite{louppe2014understanding}.
Moreover, computation of $\b{\ell}_{\mathcal{X}_b}$ and $\b{u}_{\mathcal{X}_b}$ take $\mathcal{O}(d\, n)$ time at every depth of tree, resulting in $\mathcal{O}(d\, n \log n)$ in a tree.
The complexity of a forest with $|\mathcal{F}|$ trees is, therefore, $\mathcal{O}(|\mathcal{F}|\, d\, n \log n)$.
In evaluation, every point traverses at most $\mathcal{O}(\log n)$ edges. For the $n$ points and the whole forest, it becomes $\mathcal{O}(|\mathcal{F}|\, n \log n)$. Note that if we use subsampling with size $\psi$, the complexity of training and evaluation phases become $\mathcal{O}(|\mathcal{F}|\, d\, \psi \log \psi)$ and $\mathcal{O}(|\mathcal{F}|\, n \log \psi)$, respectively. 
Also, note that since the trees are independent of one another, trees can be computed in parallel.
\end{proof}

\begin{proposition}
After construction of forest by the initial batch, the training and evaluation (score calculation) phases of online processing in iMondrian forest each take $\mathcal{O}(|\mathcal{F}|\, d\, (n+m) \log (n+m))$ and $\mathcal{O}(|\mathcal{F}|\, (n+m) \log (n+m))$ time, respectively, assuming that the trees are balanced.
\end{proposition}

\begin{proof}
In training, the $i$-th new point traverses at most $\log (n+i)$ edges in the worst case because every new point adds either an internal or leaf node to the tree. The time for the $m$ new points is $\mathcal{O}(\log(n+1) + \dots + \log(n+m)) = \mathcal{O}(\log ((n+m)!) - \log (n!)) \approx \mathcal{O}((n+m) \log (n+m)) - \mathcal{O}(n \log n) = \mathcal{O}((n+m) \log (n+m))$ where approximation is because of Stirling's approximation for the factorial function. 
Also, calculation of $\b{e}_\ell$ and $\b{e}_u$ take $\mathcal{O}(d)$ time for each new point at every node. 
The time for the whole forest is $\mathcal{O}(|\mathcal{F}|\, d\, (n+m) \log (n+m))$; although, the trees can be processed in parallel. 
In evaluation, the tree includes $(n+m)$ leaves so every point traverses at most $\mathcal{O}(\log(n+m))$ edges. Having all the trees and the $n+m$ points, including the re-evaluation of previous batch, takes $\mathcal{O}(|\mathcal{F}|\, (n+m) \log (n+m))$ time.
\end{proof}

\begin{figure*}[!t]
\centering
\includegraphics[width=5.5in]{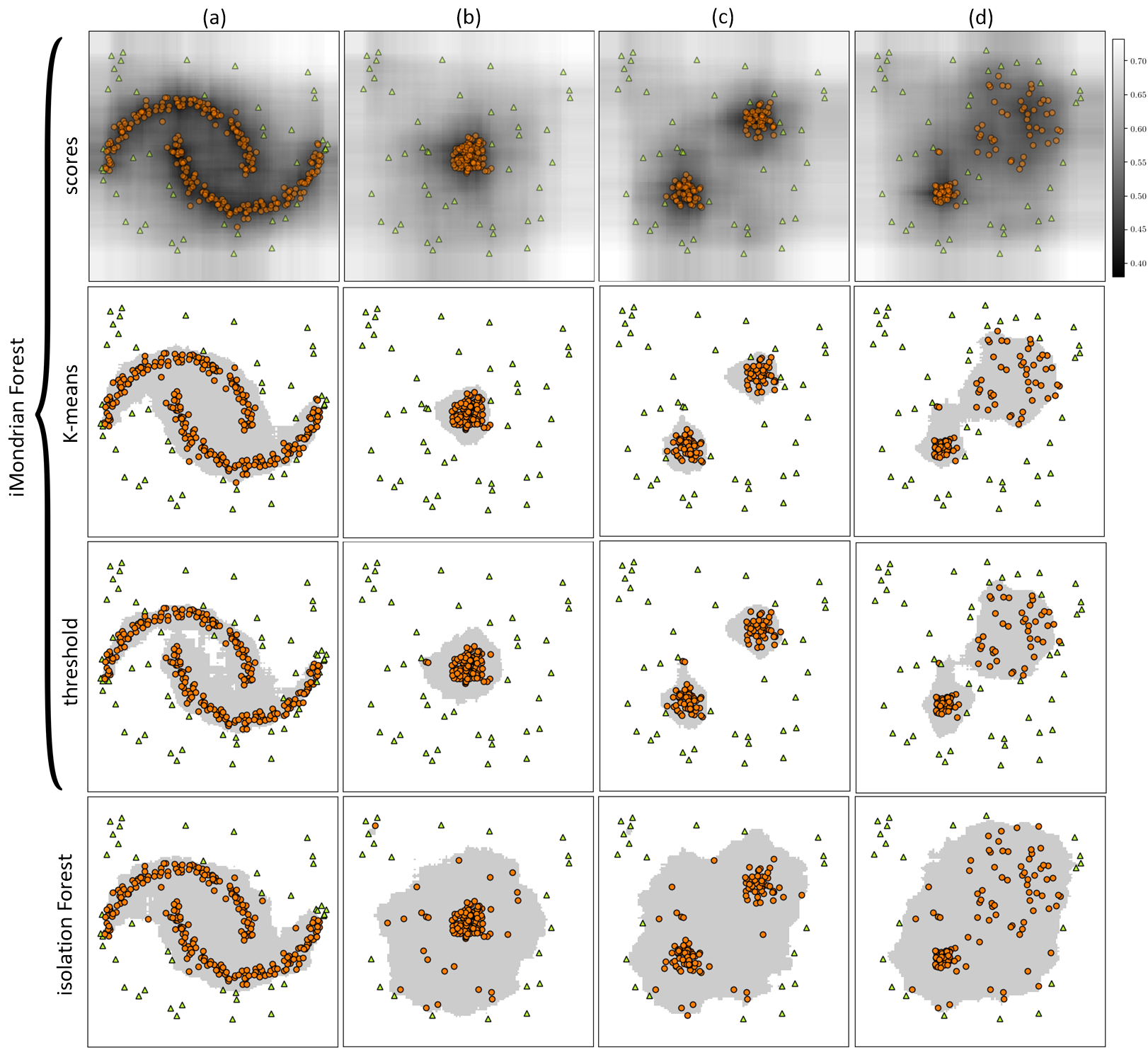}
\caption{Comparison of batch anomaly detection in iMondrian and iForests for the synthetic datasets (a)-(d). The orange circles and green triangles correspond to the detected normal and anomalous points, respectively, while shaded gray regions show the partition of space detected as normal.}
\label{figure_synthetic_batch}
\end{figure*}

\section{Experiments}

\subsection{Synthetic Data}

We created four two-dimensional synthetic datasets (a)-(d) as can be seen in Fig. \ref{figure_synthetic_batch}. The datasets include a variety of edge-case distributions of random anomalous points. 
Dataset (a) has $255$ and the rest of datasets have $100$ inliers while the number of outliers are $45$. 

\begin{figure*}[!t]
\centering
\includegraphics[width=6in]{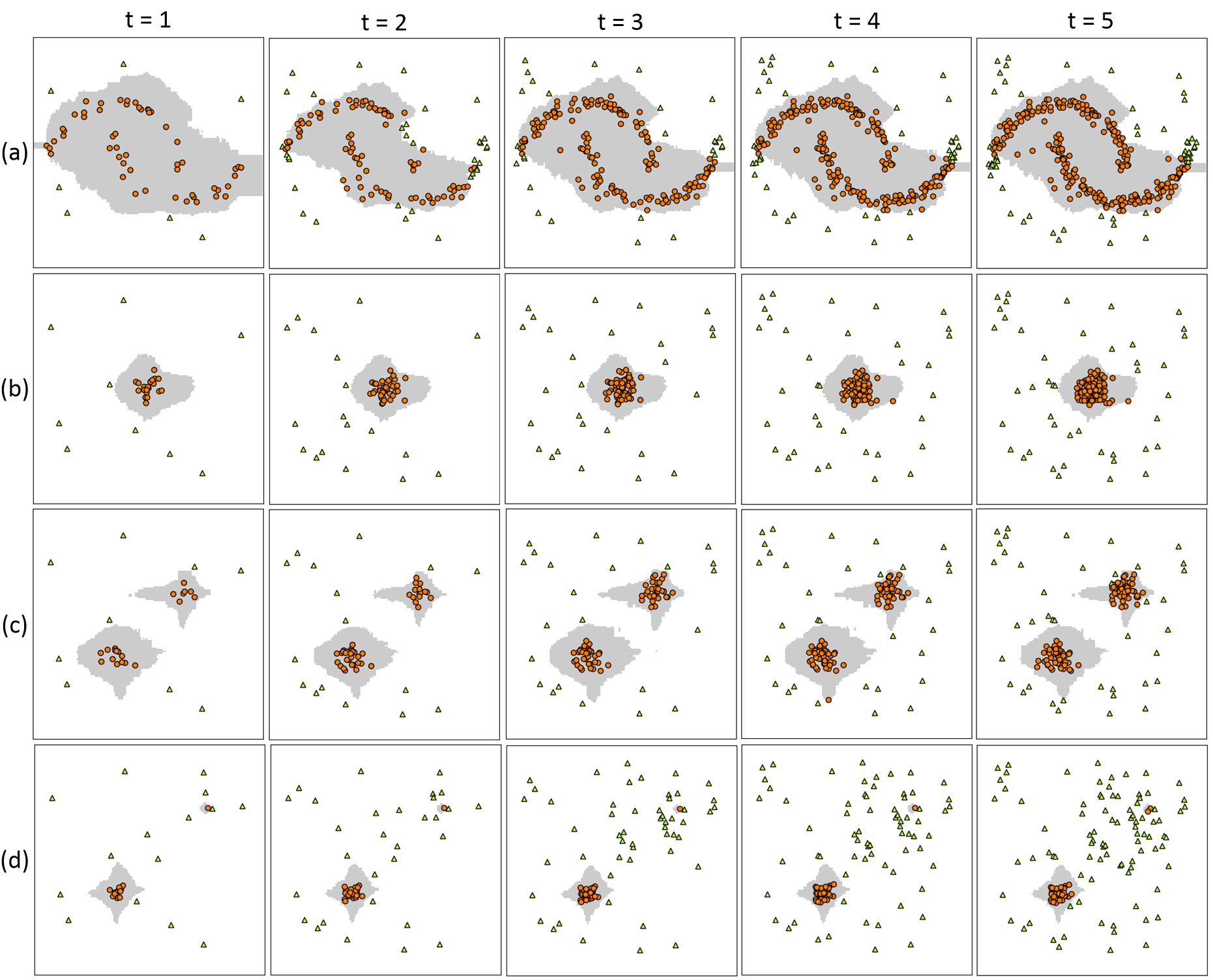}
\caption{Online anomaly detection in iMondrian forest for the synthetic datasets (a)-(d). The orange circles and green triangles correspond to the detected normal and anomalous points, respectively, while shaded gray regions show the partition of space detected as normal. Time steps are denoted by $t$ in this figure.}
\label{figure_synthetic_online}
\end{figure*}

\subsubsection{Batch Experiments}

The results of batch processing in iMondrian forest on the synthetic datasets are illustrated in Fig. \ref{figure_synthetic_batch}. The anomaly scores are shown for the input space of data. As expected, the scores are higher for anomalous points of space, which are far away from the core of distribution. 
Fig. \ref{figure_synthetic_batch} shows the results of both K-means clustering and thresholding (with threshold $s=0.5$) which perform almost equally well. 
We also show the result of iForest (with threshold $s=0.5$) for the sake of comparison. iMondrian forest clearly performed much better than iForest due to having much fewer false negatives. It is because iMondrian trees take into account the smallest blocks containing the points within a node while iForest considers the whole block.

\begin{table*}[!t]
\caption{Characteristics of utilized datasets for experiments.}
\label{table_datasets}
\begin{minipage}{\textwidth}
\renewcommand{\arraystretch}{1.3}  
\centering
\scalebox{0.9}{    
\begin{tabular}{l | c | c | c | c | c | c | c | c | c}
\hline
\hline
 & WBC & Pima & Thyroid & Satellite & Optdigits & Ionosphere & Wine & SMTP & CICIDS \\
\hline
$\#$Instances & 278 & 768 & 3772 & 6435 & 5216 & 351 & 129 & 95156 & 691406 \\
\hline
$\#$Features & 30 & 8 & 6 & 36 & 64 & 33 & 13 & 3 & 77 \\
\hline
$\%$ anomalies & 37\% & 35\% & 2.5\% & 32\% & 3\% & 36\% & 7.7\% & 0.03\% & 36\% \\
\hline
\hline
\end{tabular}%
}
\end{minipage}
\end{table*}

\subsubsection{Online Experiments}

We divided every dataset, using stratified sampling, into five subsets with equal amounts of outliers. We used these subsets to simulate streaming data by adding each subset to the existing data in a succession of five steps. 
Fig. \ref{figure_synthetic_online} shows the results of online processing using iMondrian forests on the synthetic datasets. 
K-means clustering was used in all of these experiments. In set (a), we see that in the second step, some inliers are falsely recognized as anomalous; however, by receiving more data in the next steps, the structures of iMondrian trees have been modified correctly and those points are recognized correctly as inliers. 
For set (c), merely some core points of the larger blob are detected as normal. This is because in the initial steps, there happen to be far fewer points from the smaller blob so the algorithm has found that region to be anomalous; however, if that blob had become much denser in the further steps, they would be detected as normal.
Overall, we see that for the different datasets, the performance of online iMondrian forest is acceptable. 

\subsection{Real Data}

We selected eight varied datasets with different characteristics from the outlier detection datasets \cite{web_anomaly_datasets} and one very large dataset, CICIDS 2017 \cite{sharafaldin2018toward,web_CICIDS_dataset}. 
In CICIDS data, we only used the data of Wednesday and excluded its one categorical feature.
Table \ref{table_datasets} reports the characteristics of these datasets. The datasets have different sample size, dimensionality, and percentage of outliers.

\begin{table*}[!t]
\caption{Comparison of batch anomaly detection methods. Rates are AUC percentage and times are in seconds averaged over the folds. The upward arrows in AUC rates mean iMForest outperforms the other methods.}
\label{table_experiments_anomaly_detection_batch}
\begin{minipage}{\textwidth}
\renewcommand{\arraystretch}{1.3}  
\centering
\scalebox{0.7}{    
\begin{tabular}{l | l | l | c | c | c | c | c | c | c | c }
\hline
\hline
\multicolumn{3}{c|}{} & \textbf{WBC} & \textbf{Pima} & \textbf{Thyroid} & \textbf{Satellite} & \textbf{Optdigits} & \textbf{Ionosphere} & \textbf{Wine} & \textbf{SMTP} \\ 
\hline
\multirow{4}{*}{ \textbf{iMForest}} 
& \multirow{2}{*}{ \textbf{Train:}} 
& \textbf{Time:} & 2.40 $\pm$ 0.01 & 2.54 $\pm$ 0.04 & 4.96 $\pm$ 0.15 & 16.33 $\pm$ 0.18 & 5.28 $\pm$ 0.03  & 2.27 $\pm$ 0.02 & 0.48 $\pm$ 0.00 & 68.64 $\pm$ 1.11  \\
& & \textbf{AUC:} & 86.35 $\pm$ 1.31 & 63.63 $\pm$ 1.02 & 95.36 $\pm$ 0.41 & 73.93 $\pm$ 1.19 & 72.90 $\pm$ 3.49  & 86.07 $\pm$ 0.95 & 99.01 $\pm$ 0.16 & 86.76 $\pm$ 1.47 \\
\cline{2-11}
& \multirow{2}{*}{ \textbf{Test:}} 
& \textbf{Time:} & 0.04 $\pm$ 0.00 & 0.07 $\pm$ 0.01 & 0.33 $\pm$ 0.02 & 1.58 $\pm$ 0.02 & 0.35 $\pm$ 0.00  & 0.02 $\pm$ 0.00 & 0.02 $\pm$ 0.00 & 7.42 $\pm$ 0.11 \\
& & \textbf{AUC:} & 86.25 $\pm$ 5.02 & 63.74 $\pm$ 9.39 & 95.37 $\pm$ 1.62 & 73.67 $\pm$ 2.33 & 73.00 $\pm$ 7.64  & 83.99 $\pm$ 6.32 & 99.71 $\pm$ 0.28 & 85.12 $\pm$ 14.25  \\
\hline
\hline
\multirow{4}{*}{ \textbf{iForest}} 
& \multirow{2}{*}{ \textbf{Train:}} 
& \textbf{Time:} & 0.14 $\pm$ 0.00 & 0.14 $\pm$ 0.00 & 0.25 $\pm$ 0.01 & 0.46 $\pm$ 0.01 & 0.81 $\pm$ 0.01  & 0.12 $\pm$ 0.00 & 0.08 $\pm$ 0.00 & 4.41 $\pm$ 0.05  \\
& & \textbf{AUC:} & 78.75 $\pm$ 1.62 $\b{\uparrow}$ & 67.49 $\pm$ 1.36 & 97.89 $\pm$ 0.22 & 70.13 $\pm$ 2.13 $\b{\uparrow}$ & 68.38 $\pm$ 4.64 $\b{\uparrow}$ & 84.74 $\pm$ 0.94 $\b{\uparrow}$ & 79.56 $\pm$ 10.59 $\b{\uparrow}$ & 90.74 $\pm$ 1.37  \\
\cline{2-11}
& \multirow{2}{*}{ \textbf{Test:}} 
& \textbf{Time:} & 0.01 $\pm$ 0.00 & 0.01 $\pm$ 0.00 & 0.02 $\pm$ 0.00 & 0.03 $\pm$ 0.00 & 0.03 $\pm$ 0.00  & 0.01 $\pm$ 0.00 & 0.01 $\pm$ 0.00 & 0.19 $\pm$ 0.00  \\
& & \textbf{AUC:} & 78.81 $\pm$ 6.20 $\b{\uparrow}$ & 68.00 $\pm$ 5.14 & 97.87 $\pm$ 0.84 & 70.13 $\pm$ 3.46 $\b{\uparrow}$ & 68.36 $\pm$ 8.11 $\b{\uparrow}$ & 84.32 $\pm$ 6.19 & 76.09 $\pm$ 10.11 $\b{\uparrow}$ & 89.44 $\pm$ 10.02  \\
\hline
\hline
\multirow{4}{*}{ \textbf{LOF}} 
& \multirow{2}{*}{ \textbf{Train:}} 
& \textbf{Time:} & 0.01 $\pm$ 0.00 & 0.01 $\pm$ 0.00 & 0.04 $\pm$ 0.00 & 0.65 $\pm$ 0.00 & 1.84 $\pm$ 0.05  & 0.01 $\pm$ 0.00 & 0.01 $\pm$ 0.00 & 1.05 $\pm$ 0.06  \\
& & \textbf{AUC:} & 61.12 $\pm$ 1.56 $\b{\uparrow}$ & 49.91 $\pm$ 1.41 $\b{\uparrow}$ & 70.26 $\pm$ 2.10 $\b{\uparrow}$ & 52.65 $\pm$ 0.33 $\b{\uparrow}$ & 60.84 $\pm$ 1.67 $\b{\uparrow}$ & 89.59 $\pm$ 0.81 & 98.70 $\pm$ 1.29 $\b{\uparrow}$ & 53.51 $\pm$ 7.25 $\b{\uparrow}$  \\
\cline{2-11}
& \multirow{2}{*}{ \textbf{Test:}} 
& \textbf{Time:} & 0.01 $\pm$ 0.00 & 0.01 $\pm$ 0.00 & 0.05 $\pm$ 0.00 & 0.77 $\pm$ 0.01 & 2.13 $\pm$ 0.08  & 0.01 $\pm$ 0.00 & 0.01 $\pm$ 0.00 & 1.14 $\pm$ 0.04  \\
& & \textbf{AUC:} & 61.94 $\pm$ 7.56 $\b{\uparrow}$ & 51.36 $\pm$ 7.50 $\b{\uparrow}$ & 66.17 $\pm$ 13.06 $\b{\uparrow}$ & 53.26 $\pm$ 2.32 $\b{\uparrow}$ & 61.12 $\pm$ 11.65 $\b{\uparrow}$ & 89.87 $\pm$ 7.23 & 92.58 $\pm$ 5.69 $\b{\uparrow}$ & 56.23 $\pm$ 25.90 $\b{\uparrow}$ \\
\hline
\hline
\multirow{4}{*}{ \textbf{SVM}} 
& \multirow{2}{*}{ \textbf{Train:}} 
& \textbf{Time:} & 0.03 $\pm$ 0.00 & 0.04 $\pm$ 0.00 & 0.27 $\pm$ 0.00 & 3.78 $\pm$ 0.00 & 3.10 $\pm$ 0.01  & 0.01 $\pm$ 0.00 & 0.01 $\pm$ 0.00 & 240.93 $\pm$ 3.81   \\
& & \textbf{AUC:} & 49.40 $\pm$ 3.16 $\b{\uparrow}$ & 51.93 $\pm$ 0.02 $\b{\uparrow}$ & 84.36 $\pm$ 0.53 $\b{\uparrow}$ & 48.54 $\pm$ 0.57 $\b{\uparrow}$ & 50.52 $\pm$ 3.81 $\b{\uparrow}$ & 76.23 $\pm$ 0.86 $\b{\uparrow}$ & 68.59 $\pm$ 4.25 $\b{\uparrow}$ & 84.14 $\pm$ 1.75 $\b{\uparrow}$  \\
\cline{2-11}
& \multirow{2}{*}{ \textbf{Test:}} 
& \textbf{Time:} & 0.01 $\pm$ 0.00 & 0.01 $\pm$ 0.00 & 0.01 $\pm$ 0.00 & 0.18 $\pm$ 0.00 & 0.15 $\pm$ 0.00  & 0.01 $\pm$ 0.00 & 0.01 $\pm$ 0.00 & 4.89 $\pm$ 0.09  \\
& & \textbf{AUC:} & 94.21 $\pm$ 1.35 & 60.01 $\pm$ 8.37 $\b{\uparrow}$ & 84.49 $\pm$ 4.12 $\b{\uparrow}$ & 64.04 $\pm$ 1.65 $\b{\uparrow}$ & 37.49 $\pm$ 7.41 $\b{\uparrow}$ & 76.61 $\pm$ 7.90 $\b{\uparrow}$ & 91.13 $\pm$ 3.83 $\b{\uparrow}$ & 83.06 $\pm$ 17.75 $\b{\uparrow}$ \\
\hline
\hline
\end{tabular}%
}
\end{minipage}
\end{table*}

\begin{table*}[!t]
\caption{Comparison of online anomaly detection methods. Rates are AUC percentage and times are in seconds. Upward arrows mean better performance of iMForest.}
\label{table_experiments_anomaly_detection_online}
\renewcommand{\arraystretch}{1.3}  
\centering
\scalebox{0.63}{    
\begin{tabular}{l | c | l | c | c | c | c | c | c | c || l | c | l | c | c | c | c | c | c }
\hline
\hline
& \textbf{Stages} &  \textbf{Pima} & \textbf{Thyroid} & \textbf{Satellite} & \textbf{Optdigits} & \textbf{Ionosphere} & \textbf{Wine} & \textbf{SMTP} & \textbf{CICIDS} & & \textbf{Stages} &  \textbf{Pima} & \textbf{Thyroid} & \textbf{Satellite} & \textbf{Optdigits} & \textbf{Ionosphere} & \textbf{Wine} & \textbf{SMTP} \\ 
\hline
\multirow{10}{*}{ \rotatebox[origin=c]{90}{\textbf{iMForest}} } 
& \textbf{Time:} & 1.25  &  6.59 &  13.24  &  9.85  &  0.56  &  0.21  &  185.33  & 2.6E3 &
\multirow{10}{*}{ \rotatebox[origin=c]{90}{\textbf{osPCA1}} } & \textbf{Time:} &  1.50 & 9.02  &  17.72  &  15.67  &  0.95  &  0.25  &  252.87 \\
& \textbf{AUC:} & 70.19  &  95.41 &  71.94  &  67.50  &  86.62  &  95.65  &  95.48  & 71.02 &
& \textbf{AUC:} & 80.37  & 40.42 $\b{\uparrow}$ &  26.48 $\b{\uparrow}$ &  54.66 $\b{\uparrow}$ &  69.95 $\b{\uparrow}$ &  86.95 $\b{\uparrow}$ &  10.52 $\b{\uparrow}$ \\
\cline{2-10}\cline{12-19}
& \textbf{Time:} & 1.48  &  8.40 &   15.26 & 11.88  &  0.65  &  0.22  &  364.03  & 1.0E4 &
& \textbf{Time:} & 1.58  & 9.24  &  16.17  &  16.49  &  1.20  &  0.24  &  252.88 \\
& \textbf{AUC:} & 68.07  &  94.27 &  73.73  &  68.00  &  85.07  &  98.91  &  95.01  &  70.95 &
& \textbf{AUC:} & 75.10  & 45.96 $\b{\uparrow}$ &  42.43 $\b{\uparrow}$ &  57.19 $\b{\uparrow}$ &  61.24 $\b{\uparrow}$ &  60.86 $\b{\uparrow}$ &  17.25 $\b{\uparrow}$ \\
\cline{2-10}\cline{12-19}
& \textbf{Time:} & 1.59  &  9.20 &  16.65  & 12.93  &  0.69  &  0.23  &  319.45 &  5.9E3 &
& \textbf{Time:} & 1.85  & 9.62  &  17.01  &  18.57  &  0.90  &  0.24  &  282.63 \\
& \textbf{AUC:} & 65.45  &  94.65 &  74.15  &  66.70  &  84.27  &  98.79  &  96.58  & 70.76 &
& \textbf{AUC:} & 73.26  & 53.49 $\b{\uparrow}$ &  45.65 $\b{\uparrow}$ &  54.24 $\b{\uparrow}$ &  53.67 $\b{\uparrow}$ &  62.31 $\b{\uparrow}$ &  11.41 $\b{\uparrow}$ \\
\cline{2-10}\cline{12-19}
& \textbf{Time:} &  1.72 &  10.32 &  18.87  & 14.45  &  0.73  &  0.24  &  349.33 & 7.3E3 &
& \textbf{Time:} & 1.85  & 9.33  &  18.49  &  19.32  &  0.87  &  0.25  &  302.31 \\
& \textbf{AUC:} & 64.51  &  94.58 &  74.00  &  66.40  &  83.10  &  98.64  &  93.84  & 70.80 &
& \textbf{AUC:} & 72.00  & 52.91 $\b{\uparrow}$ &  47.32 $\b{\uparrow}$ &  51.65 $\b{\uparrow}$ &  50.92 $\b{\uparrow}$ &  67.11 $\b{\uparrow}$ &  18.37 $\b{\uparrow}$ \\
\cline{2-10}\cline{12-19}
& \textbf{Time:} & 1.88  &  11.29 &  20.67  &  15.29  &  0.79  &  0.29  &  393.93  & 8.6E3 & 
& \textbf{Time:} & 1.95  & 9.54  &  19.59  &  20.16  &  0.89  &  0.28  &  321.40 \\
& \textbf{AUC:} & 65.50  &  94.64 &  73.40  &  67.10  &  82.80  &  97.60  &  92.87 & 70.83 &
& \textbf{AUC:} & 71.34  & 55.80 $\b{\uparrow}$ &  48.08 $\b{\uparrow}$ &  51.09 $\b{\uparrow}$ &  49.74 $\b{\uparrow}$ &  72.35 $\b{\uparrow}$ &  24.72 $\b{\uparrow}$ \\
\hline
\hline
\multirow{10}{*}{ \rotatebox[origin=c]{90}{\textbf{Incremental LOF}} } 
& \textbf{Time:} &  0.001 & 0.006  &  0.04  &  0.06  &  0.001  &  0.0009  &  0.71 & 4.1E2 &
\multirow{10}{*}{ \rotatebox[origin=c]{90}{\textbf{osPCA2}} } & \textbf{Time:} &  0.15  & 3.51  &  14.58  &  11.32  &  0.12  &  0.01  &  2101.3 \\
& \textbf{AUC:} & 58.81 $\b{\uparrow}$  & 85.61 $\b{\uparrow}$ &  54.23 $\b{\uparrow}$ &  56.87 $\b{\uparrow}$ &  93.06  &  95.65  &  94.90 $\b{\uparrow}$ & 46.58 $\b{\uparrow}$ & 
& \textbf{AUC:} & 50.94 $\b{\uparrow}$ & 49.93 $\b{\uparrow}$ &  49.94 $\b{\uparrow}$ &  49.95 $\b{\uparrow}$ &  48.88 $\b{\uparrow}$ &  47.82 $\b{\uparrow}$ &  49.99 $\b{\uparrow}$\\
\cline{2-10}\cline{12-19}
& \textbf{Time:} & 0.001  & 0.02  &  0.14  & 0.26 &  0.001  & $\approx$ 0   &  1.87 & 1.5E3 &
& \textbf{Time:} & 2.04  & 17.35  &  42.41  &  33.80  &  0.95  &  0.25  &  5346.3 \\
& \textbf{AUC:} & 55.13 $\b{\uparrow}$ & 70.62 $\b{\uparrow}$ &  53.76 $\b{\uparrow}$ &  60.78 $\b{\uparrow}$ &  89.57  &  98.91  &  95.44 & 46.06 $\b{\uparrow}$ &
& \textbf{AUC:} & 56.10 $\b{\uparrow}$ & 57.00 $\b{\uparrow}$ &  54.03 $\b{\uparrow}$ &  46.14 $\b{\uparrow}$ &  57.02 $\b{\uparrow}$ &  58.69 $\b{\uparrow}$ &  58.03 $\b{\uparrow}$ \\
\cline{2-10}\cline{12-19}
& \textbf{Time:} & 0.001  & 0.04  &  0.29  & 0.58 &  0.003  & $\approx$ 0   &  3.93 & 3.3E3 &
& \textbf{Time:} & 2.32  & 25.88  &  68.78  &  54.53  &  1.06  & 0.26  &  9979 \\
& \textbf{AUC:} & 53.31 $\b{\uparrow}$ & 72.34 $\b{\uparrow}$ &  52.33 $\b{\uparrow}$ &  62.98 $\b{\uparrow}$ &  88.81  &  91.06 $\b{\uparrow}$ &  58.59 $\b{\uparrow}$ & 45.99 $\b{\uparrow}$ &
& \textbf{AUC:} & 59.42 $\b{\uparrow}$ & 64.59 $\b{\uparrow}$ &  56.33 $\b{\uparrow}$ &  40.87 $\b{\uparrow}$ &  61.58 $\b{\uparrow}$ &  67.63 $\b{\uparrow}$ &  68.00 $\b{\uparrow}$ \\
\cline{2-10}\cline{12-19}
& \textbf{Time:} & 0.003  & 0.07  &  0.52  &  1.06  &  0.003  &  $\approx$ 0  &  5.39 & 4.3E3 & 
& \textbf{Time:} & 2.81  & 33.53  &  69.17  &  73.20  &  1.07  &  0.34  &  14416 \\
& \textbf{AUC:} & 49.10 $\b{\uparrow}$ & 69.61 $\b{\uparrow}$ &  51.64 $\b{\uparrow}$ &  64.22 $\b{\uparrow}$ &  88.93  &  81.92 $\b{\uparrow}$ &  52.79 $\b{\uparrow}$ & 46.00 $\b{\uparrow}$ &
& \textbf{AUC:} & 60.02 $\b{\uparrow}$ & 64.78 $\b{\uparrow}$ &  57.23 $\b{\uparrow}$ &  38.97 $\b{\uparrow}$ &  62.84 $\b{\uparrow}$ &  72.14 $\b{\uparrow}$ &  69.13 $\b{\uparrow}$ \\
\cline{2-10}\cline{12-19}
& \textbf{Time:} & 0.005  & 0.11  & 0.79   &  1.68  &  0.004  &  $\approx$ 0  &  8.02 & 7.9E3 & 
& \textbf{Time:} & 3.21  & 42.30  &  122.71  &  93.13  &  1.18  &  0.40  &  18678 \\
& \textbf{AUC:} & 48.40 $\b{\uparrow}$ & 68.10 $\b{\uparrow}$ &  51.69 $\b{\uparrow}$ &  62.75 $\b{\uparrow}$ &  90.13  &  91.51 $\b{\uparrow}$ &  49.60 $\b{\uparrow}$ & 45.93 $\b{\uparrow}$ &
& \textbf{AUC:} & 61.05 $\b{\uparrow}$ & 67.19 $\b{\uparrow}$ &  57.36 $\b{\uparrow}$ &  36.60 $\b{\uparrow}$ &  64.51 $\b{\uparrow}$ &  75.04 $\b{\uparrow}$ &  68.82 $\b{\uparrow}$ \\
\hline
\hline
\end{tabular}%
}
\end{table*}

\subsubsection{Batch Experiments}

We compared iMondrian forest with iForest, LOF (with $k=10$), and one-class SVM (with RBF kernel). 
The experiments were performed with $10$-fold cross validation except for the wine dataset where we used two folds due to small sample size.
The average Area Under ROC Curve (AUC) \cite{fawcett2006introduction}, which is a common measure in anomaly detection literature \cite{liu2008isolation,liu2012isolation}, and time of algorithms over the ten folds are reported in Table \ref{table_experiments_anomaly_detection_batch}. 
The results are reported for both training and test subsets of data. In most datasets, we outperform iForest, LOF, and SVM. 
In three datasets Pima, thyroid, and and SMTP, iForest is slightly better; although, the difference is not significant. In time, iForest is mostly better than iMondrian but its accuracy is often less.

\subsubsection{Online Experiments}

For the online experiments, we divided datasets into five stages using stratified sampling and introduce the streaming data to the algorithms where each new point is accumulated to previous data. The AUC of a stage is for scores up to that stage.
WBC was not used here because there was such a small relative portion of outliers it made the stratified sampling not possible.
We compared iMondrian forest with incremental LOF (with $k=10$), osPCA1, and osPCA2, reported in Table \ref{table_experiments_anomaly_detection_online}. 
The results of CICIDS on osPCA methods are not reported as they did not perform in a reasonable time on these datasets.
The AUC of iMondrian forest reported for every stage is the rate for recalculated scores of the available data.
In the first stage of osPCA1 and osPCA2, we used decremental PCA approach with oversampling \cite{lee2013anomaly}.
In different datasets, iMondrian forest has stable performance in different stages which shows its stability over the streaming data. In most cases, we outperform all the baseline methods. In terms of time, we outperform osPCA1 and osPCA2.

\section{Conclusion and Future Work}

In this paper, we proposed iMondrian forest for batch and online anomaly detection. It is a novel hybrid of isolation and Mondrian forests which are existing methods for batch anomaly detection and online random forest, respectively. 
The proposed method makes use of the best of two worlds of isolation-based anomaly detection and online ensemble learning. 
As a future direction, we seek to investigate whether it is worthy to develop the idea of isolation-based anomaly detection in other online tree structures such as binary space partitioning forest \cite{fan2019binary} which is based on the binary partitioning process \cite{fan2018binary}.
Another future work is to investigate other clustering methods instead of K-means for clustering scores.

\bibliographystyle{IEEEtran}
\bibliography{references}

\end{document}